\documentclass[12pt]{article}

\usepackage{times}
\usepackage{amssymb}
\usepackage{amsmath}
\usepackage{epsfig}

\usepackage{algorithm}
\usepackage{algorithmic}
\usepackage{epsfig}
\usepackage{amsmath,amsthm,amsfonts,amssymb,mathrsfs}

\usepackage{epsfig}
\usepackage{amsmath}
\usepackage{graphicx}
\usepackage{booktabs}
\usepackage{euscript}
\usepackage{exscale,relsize}

\newcommand{\email}[1]{\href{mailto:#1}{\nolinkurl{#1}}}
\PassOptionsToPackage{normalem}{ulem}
\usepackage{ulem}

\renewcommand{\leq}{\ensuremath{\leqslant}}
\renewcommand{\geq}{\ensuremath{\geqslant}}

\newcommand{\beq}{\begin{equation}}
\newcommand{\eeq}{\end{equation}}

\newcommand{\prox}{\ensuremath{\text{\rm prox}}}

\newcommand{\card}{\ensuremath{\text{\rm card}}}

\newcommand{\argmin}{\ensuremath{\text{\rm argmin}\,}}

\newtheorem{theorem}{Theorem}

\newtheorem{example}[theorem]{Example}
\newcommand{\bex}{\begin{example}}
\newcommand{\eex}{\end{example}} 
\newtheorem{definition}[theorem]{Definition}

\newtheorem{lemma}[theorem]{Lemma}

\newtheorem{proposition}[theorem]{Proposition}
\newtheorem{remark}[theorem]{Remark}

\numberwithin{equation}{section}
\DeclareMathOperator{\supp}{supp}
\let\inf\relax \DeclareMathOperator*\inf{\vphantom{p}inf}
\let\sup\relax \DeclareMathOperator*\sup{\vphantom{p}sup}

\newcommand{\NN}[1]{{\mathbb{N}_{#1}}}

\newcommand{\N}{{\mathbb N}}
\newcommand{\R}{{\mathbb R}}

\newcommand{\G}{{\mathcal{G}}}

\newcommand{\bea}{\begin{eqnarray}}
\newcommand{\eea}{\end{eqnarray}}

\newcommand{\lam}{{\lambda}}

 \newcommand{\lb}{{\langle}}
\newcommand{\rb}{{\rangle}} 

\def\boldf#1{\hbox{\rlap{$#1$}\kern.4pt{$#1$}}}

 \newcommand{\trans}{^{\scriptscriptstyle
\top}}

\newcommand{\Ik}{{I_k}}

\begin{document}

\title{\LARGE Fitting Spectral Decay with the $k$-Support Norm}

\author{Andrew M. McDonald$^1$ \and Massimiliano Pontil$^{1,2}$ \and Dimitris Stamos$^{2}$ \\ \\
 {(1)} Department of Computer Science \\ 
 University College London \\
 {\em email:~\{a.mcdonald,d.stamos.12\}@ucl.ac.uk} \\
 Gower Street, London WC1E 6BT, UK \\ \\
 {(2)} Istituto Italiano di Tecnologia \\  
 Via Morego 30, 16163 Genova, Italy
}

\maketitle

\begin{abstract}
The spectral $k$-support norm enjoys good estimation properties in low rank matrix learning problems, empirically outperforming the trace norm. Its unit ball is the convex hull of rank $k$ matrices with unit Frobenius norm.
In this paper we generalize the norm to the spectral $(k,p)$-support norm, whose additional parameter $p$ can be used to tailor the norm to the decay of the spectrum of the underlying model.
We characterize the unit ball and we explicitly compute the norm.
We further provide a conditional gradient method to solve regularization problems with the norm, and we derive an efficient algorithm to compute the Euclidean projection on the unit ball in the case $p=\infty$. 
In numerical experiments, we show that allowing $p$ to vary significantly improves performance over the spectral $k$-support norm on various matrix completion benchmarks, and better captures the spectral decay of the underlying model. 
\end{abstract}

{\bfseries Keywords.} $k$-support norm, orthogonally invariant norms, matrix completion, multitask learning, proximal point algorithms.

\section{Introduction}
\label{sec:intro}
The problem of learning a sparse vector or a low rank matrix has generated much interest in recent years.
A popular approach is to use convex regularizers which encourage sparsity, and a number of these have been studied with applications including image denoising, collaborative filtering and multitask learning, see for example,
\cite{Buehlmann2011,Wainwright2014} and references therein.

Recently, the \emph{$k$-support norm} was proposed by \cite{Argyriou2012}, motivated as a tight relaxation of the set of $k$-sparse vectors of unit Euclidean norm. 
The authors argue that as a regularizer for sparse vector estimation, the norm empirically outperforms the Lasso \cite{Tibshirani1996} and Elastic Net \cite{Zou2005} penalties. 
Statistical bounds on the Gaussian width of the $k$-support norm  have been provided by \cite{Chatterjee2014}.
The $k$-support norm has also been extended to the matrix setting. 
By applying the norm to the vector of singular values of a matrix, \cite{McDonald2014a} obtain the orthogonally invariant \emph{spectral $k$-support norm}, reporting state of the art performance on matrix completion benchmarks.  

Motivated by the performance of the $k$-support norm in sparse vector and matrix learning problems, 
in this paper we study a natural generalization by considering the $\ell_p$-norms (for $p \in [1,\infty]$) in place of the Euclidean norm.  
These allow a further degree of freedom when fitting a model to the underlying data. We denote the ensuing norm the \emph{$(k,p)$-support norm}.
As we demonstrate in numerical experiments, 
using $p=2$ is not necessarily the best choice in all instances. 
By tuning the value of $p$ the model can incorporate prior information regarding the singular values. 
When prior knowledge is lacking, the parameter can be chosen by validation, hence the model can adapt to a variety of decay patterns of the singular values.  
An interesting property of the norm is that it interpolates between the $\ell_1$ norm (for $k=1$) and the $\ell_p$-norm (for $k=d$). 
It follows that varying both $k$ and $p$ the norm allows one to learn sparse vectors which exhibit different patterns of decay in the non-zero elements. 
In particular, when $p=\infty$ the norm prefers vectors which are constant.

A main goal of the paper is to study the proposed norm in matrix learning problems. 
The $(k,p)$-support norm is a symmetric gauge function hence it induces the orthogonally invariant \emph{spectral $(k,p)$-support norm}.   
This interpolates between the trace norm (for $k=1$) and the Schatten $p$-norms (for $k=d$) and its unit ball has a simple geometric interpretation as the convex hull of matrices of rank no greater than $k$ and Schatten $p$-norm no greater than one. 
This suggests that the new norm favors low rank structure and the effect of varying $p$ allows different patterns of decay in the spectrum.
In the special case of $p=\infty$, the $(k,p)$-support norm is the dual of the Ky-Fan $k$-norm \cite{Bhatia1997} and it encourages a flat spectrum when used as a regularizer.  

The main contributions of the paper are: i) we propose the $(k,p)$-support norm as an extension of the $k$-support norm and we characterize in particular the unit ball of the induced orthogonally invariant matrix norm (Section \ref{sec:kp-sup}); ii) we show that the norm can be computed efficiently and we discuss the role of the parameter $p$ (Section \ref{sec:computing-the-norm}); iii) we outline a conditional gradient method to solve the associated regularization problem for both vector and matrix problems (Section \ref{sec:optimization});
and in the special case $p=\infty$ we provide an $\mathcal{O}(d \log d)$ computation of the projection operator (Section \ref{sec:prox}); finally, iv) we present numerical experiments on matrix completion benchmarks which demonstrate that the proposed norm offers significant improvement over previous methods, and we discuss the effect of the parameter $p$ (Section \ref{sec:experiments}).
The appendix contains derivations of results which are sketched in or are omitted from the main body of the paper.


{\bf Notation.}
We use $\NN{n}$ for the set of integers from $1$ up to and including $n$. 
We let $\R^d$ be the $d$-dimensional real vector space, whose elements are denoted by lower case letters. 
For any vector $w\in \R^d$, its {\em support} is defined as $\supp(w) = \{i \in \NN{d}: w_i \neq 0\} $, and its \emph{cardinality} is defined as $\card(w) = \vert \supp(w) \vert$. 
We let $\R^{d \times m}$ be the space of $d \times m$ real matrices.
%
%
We denote the rank of a matrix as $\textrm{rank}(W)$. 
We let $\sigma(W) \in \R^r$ be the vector formed by the singular values of $W$, where $r=\min(d,m)$, and where we assume that the singular values are ordered nonincreasing, that is 
$\sigma_1(W) \geq \cdots \geq \sigma_r(W) \geq 0$.
For $p\in[1,\infty)$ the $\ell_p$-norm of a vector $w \in \R^d$ is defined as $\|w\|_p = ( \sum_{i=1}^d |w_i|^p)^{1/p}$ and $\|w\|_\infty = \max_{i=1}^d |w_i|$.
Given a norm $\Vert \cdot \Vert$ on $\R^d$ or $\R^{d \times m}$, $\Vert \cdot \Vert_*$ denotes the corresponding dual norm, defined by $\Vert u \Vert_* = \sup \{ \langle u,w \rangle : \Vert w \Vert \leq 1  \}$. 
The convex hull of a subset $S$ of a vector space is denoted $\textrm{co}(S)$.


\section{Background and Previous Work}
\label{sec:background}
\vspace{-.1truecm}
For every $k \in \N_d$, the $k$-support norm $\Vert \cdot \Vert_{(k)}$ is defined as the norm whose unit ball is given by 
\begin{align}
\text{co}\left\{ w \in \R^d: \card(w) \leq k, \Vert w \Vert_2 \leq 1 \right\},
\label{eqn:ksup-unit-ball}
\end{align}
that is, the convex hull of the set of vectors of cardinality at most $k$ and $\ell_2$-norm no greater than one \cite{Argyriou2012}. 
We readily see that for $k=1$ and $k=d$ we recover the unit ball of the $\ell_1$ and $\ell_2$-norms respectively.

The $k$-support norm of a vector $w\in \R^d$ can be expressed as an infimal convolution \cite[p.~34]{Rockafellar1970},
\begin{align}
\|w\|_{(k)}=\inf_{(v_g)}  \Bigg\{ \sum_{g \in \G_k} \Vert v_g \Vert_2 :  \sum_{g \in \G_k} v_g = w \Bigg\},
\label{eqn:GLO}
\end{align}
where $\G_k$ is the collection of all subsets of $\NN{d}$ containing at most $k$ elements and the infimum is over all vectors $v_g\in \R^d$ such that $
\supp(v_g) \subseteq g$, for $g \in \G_k$. 
Equation \eqref{eqn:GLO} highlights that the $k$-support norm is a special case of the group lasso with overlap \cite{Jacob2009-GL}, where the cardinality of the support sets is at most $k$.  
This expression suggests that when used as a regularizer, the norm encourages vectors $w$ to be a sum of a limited number of vectors with small support.  
Due to the variational form of \eqref{eqn:GLO} computing the norm is not straightforward, however \cite{Argyriou2012} note that the dual norm has a simple form, namely it is the $\ell_2$-norm of the $k$ largest components,
\begin{align}
\Vert u \Vert_{(k),*} &= \sqrt{\sum_{i=1}^k (\vert u \vert^{\downarrow}_i)^2},~~~u \in \R^d\label{eqn:ksup-dualeq},
\end{align} 
where $|u|^{\downarrow}$ is the vector obtained from $u$ by reordering its components so that they are nonincreasing in absolute value. 
Note also from equation \eqref{eqn:ksup-dualeq} that for $k=1$ and $k=d$, the dual norm is equal to the $\ell_{\infty}$-norm and $\ell_2$-norm, respectively, which agrees with our earlier observation regarding the primal norm.

A related problem which has been studied in recent years is learning a matrix from a set of linear measurements, in which the underlying matrix is assumed to have sparse spectrum (low rank). 
The trace norm, the $\ell_1$-norm of the singular values of a matrix, has been shown to perform well in this setting, see e.g.  \cite{Argyriou2008,Jaggi2010}.
%
%
Recall that a norm $\Vert\cdot \Vert$ on $\mathbb{R}^{d \times m}$ is called orthogonally invariant if 
$\Vert W \Vert= \Vert U W V \Vert$,
for any orthogonal matrices $U \in \mathbb{R}^{d \times d}$ and $V \in \mathbb{R}^{m \times m}$. 
A classical result by von Neumann establishes that a norm is orthogonally invariant if and only if it is of the form $\Vert W \Vert = g(\sigma(W))$, where $\sigma(W)$ is the vector formed by the singular values of $W$ in nonincreasing order, and $g$ is a symmetric gauge function \cite{VonNeumann1937}.
In other words, $g$ is a norm which is invariant under permutations and sign changes of the vector components, that is 
$g(w) = g(P w)= g(J w)$,
where $P$ is any permutation matrix and $J$ is diagonal with entries equal to $\pm 1$ \cite[p. 438]{Horn1991}. 

Examples of symmetric gauge functions are the $\ell_p$ norms for $p \in [1,\infty]$ and the corresponding orthogonally invariant norms are called the Schatten $p$-norms \cite[p. 441]{Horn1991}. 
In particular, those include the trace norm and Frobenius norm for $p=1$ and $p=2$ respectively.
Regularization with Schatten $p$-norms has been previously studied by \cite{Argyriou2007} and a statistical analysis has been performed by \cite{Rohde2011}.
As the set $\G_k$ includes all subsets of size $k$, expression \eqref{eqn:GLO} for the $k$-support norm reveals that is a symmetric gauge function.  
\cite{McDonald2014a} use this fact to introduce the spectral $k$-support norm for matrices, by defining $\Vert W \Vert_{(k)} = \Vert \sigma(W) \Vert_{(k)}$, for $W \in \R^{d \times m}$ and report state of the art performance on matrix completion benchmarks. 



\section{The $(k,p)$-Support Norm}
\label{sec:kp-sup}
\vspace{-.1truecm}
In this section we introduce the $(k,p)$-support norm as a natural extension of the $k$-support norm. 
This follows by applying the $\ell_p$-norm, rather than the Euclidean norm, in the infimum convolution definition of the norm.  

\begin{definition}
\label{def:k-sup-p-norm}
Let $k \in \N_d$ and $p \in [1,\infty]$. The $(k,p)$-support norm of a vector $w \in \R^d$ is defined as 
\begin{align}
\Vert w \Vert_{(k,p)} &= \inf_{(v_g)} \left\{ \sum_{g \in \G_k} \Vert v_g \Vert_p :  \sum_{g \in \G_k} v_g = w \right\}. 
\label{eqn:kp-sup-infimum-convolution}
\end{align}
where the infimum is over all vectors $v_g\in \R^d$ such that $ \supp(v_g) \subseteq g$, for $g \in \G_k$.
\end{definition}

Let us note that the norm is well defined. Indeed, positivity, homogeneity and non degeneracy are immediate. 
To prove the triangle inequality, let $w,w' \in \R^d$.  
For any $\epsilon >0$ there exist $\{ v_g \}$ and $\{ v'_{g}\}$ such that $w=\sum_{g} v_g$, $w'=\sum_{g} v'_{g}$, $\sum_{g} \Vert v_g \Vert_p  \leq \Vert w \Vert_{(k,p)} + \epsilon/2$, and $\sum_{g} \Vert v'_g \Vert_p  \leq \Vert w' \Vert_{(k,p)} + \epsilon/2$.
As $\sum_{g} v_g + \sum_{g} v'_{g} = w+w'$, we have
\begin{align*}
\Vert w+w' \Vert_{(k,p)}
&\leq \sum_{g} \Vert v_g \Vert_p + \sum_{g} \Vert v'_{g} \Vert_p  \\
&\leq \Vert w \Vert_{(k,p)} + \Vert w'\Vert_{(k,p)} + \epsilon,
\end{align*}
and the result follows by letting $\epsilon$ tend to zero. 

Note that, since a convex set is equivalent to the convex hull of its extreme points, Definition \ref{def:k-sup-p-norm} implies that the unit ball of the $(k,p)$-support norm, denoted by $C_k^p$, is given by the convex hull of the set of vectors with cardinality no greater than $k$ and $\ell_p$-norm no greater than 1, that is 
\begin{align}
C_k^p=\textrm{co}\left\{w\in \R^d:  \card(w) \leq k, \Vert w \Vert_p \leq 1 \right\}. 
\label{eq:unitball}
\end{align}

Definition \ref{def:k-sup-p-norm} gives the norm as the solution of a variational problem. 
Its explicit computation is not straightforward in the general case, however for $p=1$ the unit ball \eqref{eq:unitball} does not depend on $k$ and is always equal to the $\ell_1$ unit ball. 
Thus, the $(k,1)$-support norm is always equal to the $\ell_1$-norm, and we do not consider further this case in this section. 
Similarly, for $k=1$ we recover the $\ell_1$-norm for all values of $p$. 
For $p=\infty$, from the definition of the dual norm it is not difficult to show that $\Vert \cdot \Vert_{(k,p)} = \max \{\Vert \cdot \Vert_{\infty}, \Vert \cdot \Vert_1 / k \}$.  
We return to this in Section \ref{sec:computing-the-norm} when we describe how to compute the norm for all values of $p$.

Note further that in Equation \eqref{eqn:kp-sup-infimum-convolution}, as $p$ tends to $\infty$, the $\ell_p$-norm of each $v_g$ is increasingly dominated by the largest component of $v_g$. 
As the variational formulation tries to identify vectors $v_g$ with small aggregate $\ell_p$-norm, this suggests that higher values of $p$ encourage each $v_g$ to tend to a vector whose $k$ entries are equal.  
In this manner varying $p$ allows us adjust the degree to which the components of vector $w$ can be clustered into (possibly overlapping) groups of size $k$.

As in the case of the $k$-support norm, the dual $(k,p)$-support norm has a simple expression.  
Recall that the dual norm of a vector $u \in \R^d$ is defined by the optimization problem
\begin{align}
\label{eq:dual}
\|u\|_{(k,p),*} = \max\left\{\lb u,w\rb : \|w\|_{(k,p)} = 1 \right\}.
\end{align}

\begin{proposition}\label{prop:k-sup-p-norm-and-dual}
If $p \in (1,\infty]$ then the dual $(k,p)$-support norm is given by 
\begin{align*}
\Vert u \Vert_{(k,p),*} = \left(\sum_{i \in \Ik} |u_i|^q\right)^\frac{1}{q},~~~ u \in \R^d,
\end{align*}
where $q=p/(p-1)$ and $\Ik \subset \N_d$ is the set of indices of the $k$ largest components of $u$ in absolute value. 
Furthermore, if $p \in (1,\infty)$ and $u \in \R^d \backslash\{0\}$ then the maximum in 
\eqref{eq:dual} is attained for
\begin{equation}
w_i =
\begin{cases}
{\rm sign}(u_i)  \left(\frac{|u_i|}{\|u\|_{(k,p),*}}\right)^{\frac{1}{p-1}} & \text{if } i \in \Ik,\\
0 & \text{otherwise}.
\end{cases}\label{eqn:dual-component}
\end{equation}
If $p =\infty$ the maximum is attained for
\begin{equation*}
w_i =
\begin{cases}
{\rm sign}(u_i) & \text{if } i \in \Ik, u_i \neq 0,\\
\lambda_i \in [-1,1]  & \text{if } i \in \Ik, u_i =0 ,\\
0 & \text{otherwise}.
\end{cases}
\end{equation*}
\end{proposition}

Note that 
for $p=2$ we recover the dual of the $k$-support norm in \eqref{eqn:ksup-dualeq}. 

\subsection{The Spectral $(k,p)$-Support Norm}
From Definition \ref{def:k-sup-p-norm} it is clear that the $(k,p)$-support norm is a symmetric gauge function. 
This follows since $\G_k$ contains all groups of cardinality $k$ and the $\ell_p$-norms only involve absolute values of the components.
Hence we can define the spectral $(k,p)$-support norm as 
\begin{align*}
\Vert W \Vert_{(k,p)} = \Vert \sigma(W) \Vert_{(k,p)}, ~~~ W \in \R^{d \times m}.
\end{align*}

Since the dual of any orthogonally invariant norm is given by $\|\cdot\|_* = \Vert \sigma(\cdot) \Vert_*$, see e.g. \cite{Lewis1995}, we conclude that the dual spectral $(k,p)$-support norm is given by 
\begin{align*}
\Vert Z \Vert_{(k,p),*} = \Vert \sigma(Z) \Vert_{(k,p),*}, ~~~ Z \in \R^{d \times m}.
\end{align*}

The next result characterizes the unit ball of the spectral $(k,p)$-support norm.  
Due to the relationship between an orthogonally invariant norm and its corresponding symmetric gauge function, we see that the cardinality constraint for vectors generalizes in a natural manner to the rank operator for matrices.

\begin{proposition}
\label{prop:unit-ball-of-spectral}
The unit ball of the spectral $(k,p)$-support norm is the convex hull of the set of matrices of rank at most $k$ and Schatten $p$-norm no greater than one. 
\end{proposition}

In particular, if $p=\infty$, the dual vector norm is given by $u \in \R^d$, by $\Vert u \Vert_{(k,\infty),*} = \sum_{i=1}^k \vert u \vert^{\downarrow}_i$. Hence, for any $Z\in \R^{d \times m}$, the dual spectral norm is given by
$\Vert Z \Vert_{(k,\infty),*} = \sum_{i=1}^k  \sigma_i(Z)$, that is the sum of the $k$ largest singular values, which is also known as the Ky-Fan $k$-norm, see e.g. \cite{Bhatia1997}.


\section{Computing the Norm}
\label{sec:computing-the-norm}
\vspace{-.1truecm}
In this section we compute the norm, illustrating how it interpolates between the $\ell_1$ and $\ell_p$-norms. 

\begin{theorem}\label{thm:kp-norm}
Let $p \in (1,\infty)$. For every $w \in \mathbb{R}^d$, and $k \leq d$, it holds that
\begin{align}
\Vert w \Vert_{(k,p)} =  \Bigg[ \sum_{i=1}^\ell  (\vert w\vert^{\downarrow}_i )^p  + 
 \left( \frac{ \sum_{i=\ell+1}^d \vert w\vert^{\downarrow}_i }{\sqrt[q]{k-\ell}}\right)^p \Bigg]^{\frac{1}{p}}
\label{eq:111kp}
\end{align}
where 
$\frac{1}{p}+\frac{1}{q}=1$, 
and for $k=d$, we set $\ell=d$, otherwise $\ell$ is the largest integer in $\{0, \ldots, k-1\}$ satisfying
\begin{align}
(k-\ell)\vert w\vert^{\downarrow}_{\ell} \geq \sum_{i=\ell+1}^d \vert w\vert^{\downarrow}_{i}.
\label{eq:222kp}
\end{align}
Furthermore, the norm can be computed in $\mathcal{O}(d \log d)$ time. 
\end{theorem}
\begin{proof}
Note first that in \eqref{eq:111kp} when $\ell=0$ we understand the first term in the right hand side to be zero, and when $\ell=d$ we understand the second term to be zero. 

We need to compute
$$
\|w\|_{(k,p)}=\max \left\{\sum_{i=1}^d u_i w_i : \|u\|_{(k,p),*}\leq 1\right\}
$$
where the dual norm $\|\cdot\|_{(k,p),*}$ is described in Proposition \ref{prop:k-sup-p-norm-and-dual}. 
Let $z_i = \vert w\vert^{\downarrow}_i$. The problem is then equivalent to 
\begin{align}
\max \left\{\sum_{i=1}^d z_i u_i : \sum_{i=1}^k u_i^q \leq 1, u_1\geq \cdots \geq u_d \right\}.
\label{eq:333kp}
\end{align}
This further simplifies to the $k$-dimensional problem
\begin{align*}
\max \left\{\sum_{i=1}^{k-1} u_i z_i + u_k \sum_{i=k}^{d} z_i: \sum_{i=1}^k u_i^q \leq 1, u_1\geq \cdots \geq u_k \right\}.
\end{align*}
Note that when $k=d$, the solution is given by the dual of the $\ell_q$-norm, that is the $\ell_p$-norm.  
For the remainder of the proof we assume that $k<d$.
We can now attempt to use Holder's inequality, which states that for all vectors $x$ such that $\|x\|_q=1$,
$\lb x,y\rb \leq \|y\|_p$, and the inequality is tight if and only if
\begin{align*}
x_i = \left(\frac{|y_i|}{\|y\|_p}\right)^{p-1}{\rm sign}(y_i). 
\end{align*}
We use it for the vector $y = (z_1,\dots,z_{k-1},\sum_{i=k}^{d} z_i)$. 
The components of the maximizer $u$ satisfy $u_i = \left(\frac{z_i}{M_{k-1}}\right)^{p-1}$ if $i\leq k-1$, and 
\begin{align*}
u_{k} =  \left(\frac{\sum_{i=\ell+1}^d z_i}{M_{k-1}}\right)^{p-1}.
\end{align*}
where for every $\ell \in \{0,\dots,k-1\}$, $M_\ell$ denotes the r.h.s. in equation \eqref{eq:111kp}.
We then need to verify that the ordering constraints are satisfied. 
This requires that
\begin{align*}
(z_{k-1})^{p-1} \geq \left(\sum_{i=k}^{d} z_i\right)^{p-1}
\end{align*}
which is equivalent to inequality \eqref{eq:222kp} for $\ell=k-1$. 
If this inequality is true we are done, otherwise we set $u_{k} = u_{k-1}$ and solve the smaller problem
\begin{eqnarray}
\nonumber
\max \bigg\{\sum_{i=1}^{k-2} u_i z_i + u_{k-1}\hspace{-.1truecm} \sum_{i=k-1}^{d} \hspace{-.1truecm} z_i~:~~~~~~~~~~~~~~~~~~~~~~~~ \\
~~~~~~~~~~\sum_{i=1}^{k-2} u_i^q + 2 u_{k-1}^q \leq 1,~~~ u_1\geq \cdots \geq u_{k-1} \bigg\}.
\nonumber
\end{eqnarray}
We use again H\"older's inequality and keep the result if the ordering constraints are fulfilled. 
Continuing in this way, the generic problem we need to solve is
\begin{eqnarray}
\nonumber
\max \bigg\{\sum_{i=1}^{\ell} u_i z_i + u_{\ell+1} \sum_{i=\ell+1}^{d} z_i~:~~~~~~~~~~~~~~~~~~~~~~~~~~~\\ 
~~~~~~~~~ \sum_{i=1}^{\ell} u_i^q + (k-\ell) u_{\ell+1}^q \leq 1,~~~ u_1\geq \cdots \geq u_{\ell+1} \bigg\}
\nonumber
\end{eqnarray}
where $\ell \in \{0,\dots,k-1\}$.
Without the ordering constraints the maximum, $M_\ell$, 
is obtained by the change of variable 
$u_{\ell+1} \mapsto (k-\ell)^{\frac{1}{q}} u_{\ell}$
followed by applying H\"older's inequality. 
A direct computation provides that the maximizer is $u_i = \left(\frac{z_i}{M_\ell}\right)^{p-1}$ if $i \leq \ell$, and 
\begin{align*}
(k-\ell)^\frac{1}{q}u_{\ell+1} =  \left(\frac{\sum_{i=\ell+1}^d z_i}{(k-\ell)^\frac{1}{q} M_\ell^p}\right)^{p-1}.
\end{align*}
Using the relationship $\frac{1}{p}+\frac{1}{q}=1$, we can rewrite this as
\begin{align*}
u_{\ell+1} = \left(\frac{\sum_{i=\ell+1}^d z_i}{(k-\ell) M_\ell^p}\right)^{p-1}.
\end{align*}
Hence, the ordering constraints are satisfied if 
\begin{align*}
z_\ell^{p-1} \geq \left(\frac{\sum_{i=\ell+1}^d z_i}{(k-\ell)}\right)^{p-1},
\end{align*} 
which is equivalent to \eqref{eq:222kp}.
Finally note that 
$M_\ell$ is a nondecreasing function of $\ell$. 
This is because the problem with a smaller value of $\ell$ is more constrained, namely, it solves \eqref{eq:333kp} with the additional constraints $u_{\ell+1} = \cdots = u_d$. 
Moreover, if the constraint \eqref{eq:222kp} holds for some value $\ell \in \{0,\dots,k-1\}$ then it also holds for a smaller value of $\ell$, hence we maximize the objective by choosing the largest $\ell$. 

The computational complexity stems from using the monotonicity of $M_{\ell}$ with respect to $\ell$, which allows us to identify the critical value of $\ell$ using binary search. 
\end{proof}

Note that for $k=d$ we recover the $\ell_p$-norm and for $p =2$ we recover the result in \cite{Argyriou2012,McDonald2014a}, however our proof technique is different from theirs.

\begin{remark}[Computation of the norm for $p \in \{1,\infty\}$] 
Since the norm $\|\cdot\|_{(k,p)}$ computed above for $p \in (1,\infty)$ is continuous in $p$, the special cases $p=1$ and $p=\infty$ can be derived by a limiting argument. We readily see that for $p=1$ the norm does not depend on $k$ and it is always equal to the $\ell_1$-norm, in agreement with our observation in the previous section. 
For $p = \infty$ we obtain that
$\Vert w \Vert_{(k,\infty)} = \max \left(\|w\|_\infty, {\|w\|_1}/{k}\right)$.
\end{remark}

\section{Optimization}
\label{sec:optimization}
\vspace{-.1truecm}
In this section, we describe how to solve regularization problems using the vector and matrix $(k,p)$-support norms. We consider the constrained optimization problem 
\begin{align}
\min \left\{ f(w) : \Vert w \Vert_{(k,p)} \leq \alpha \right\} \label{eqn:ivanov},
\end{align}
where $w$ is in $\R^d$ or $\R^{d \times m}$, $\alpha >0$ is a regularization parameter and the error function $f$ is assumed to be convex and continuously differentiable. 
For example, in linear regression a valid choice is the square error, $f(w) = \|Xw-y\|_2^2$, where $X$ is matrix of observations and $y$ a vector of response variables. 
Constrained problems of form \eqref{eqn:ivanov} are also referred to as Ivanov regularization in the inverse problems literature \cite{Ivanov1978}. 

A convenient tool to solve problem \eqref{eqn:ivanov} is provided by the \emph{Frank-Wolfe} method \cite{Frank1956}, see also \cite{Jaggi2013} for a recent account. 
\begin{algorithm}[t]
\caption{Frank-Wolfe. \label{alg:FW}}
\begin{algorithmic} 
\STATE Choose $w^{(0)}$ such that $\Vert w^{(0)} \Vert_{(k,p)} \leq \alpha$
\FOR{$t=0, \ldots, T$ } \STATE{
Compute $g:= \nabla f(w^{(t)})$ \\
Compute $s:= \argmin \left\{ \langle s, g)\rangle: \Vert s \Vert_{(k,p)} \leq \alpha \right\}$  \\ 
Update $w^{(t+1)} := (1-\gamma) w^{(t)} + \gamma s$, for $\gamma := \frac{2}{t+2}$}
\ENDFOR
\end{algorithmic}
\end{algorithm}
The method is outlined in Algorithm \ref{alg:FW}, and it has worst case convergence rate $\mathcal{O}(1/T)$.
The key step of the algorithm is to solve the subproblem
\beq
\argmin \left\{ \langle s, g\rangle:\Vert s \Vert_{(k,p)} \leq \alpha \right\},
\label{eq:keystep}
\eeq
where $g=\nabla f(w^{(t)})$, that is the gradient of the objective function at the $t$-th iteration. This problem involves computing a subgradient of the dual norm at $g$. 
It can be solved exactly and efficiently as a consequence of Proposition \ref{prop:k-sup-p-norm-and-dual}. 
We discuss here the vector case and postpone the discussion of the matrix case to Section \ref{sec:MP}. 
By symmetry of the $\ell_p$-norm, problem \eqref{eq:keystep} can be solved in the same manner as the maximum in Proposition \ref{prop:k-sup-p-norm-and-dual}, and the solution is given by $s_i = -\alpha w_i$, where $w_i$ is given by \eqref{eqn:dual-component}. 
Specifically, letting $\Ik \subset \NN{d}$ be the set of indices of the $k$ largest components of $g$ in absolute value, for $p \in (1,\infty)$ we have 
\begin{equation}
\label{eq:FWp}
s_i = 
\begin{cases}
- \alpha \, {\rm sign}(g_i) \left(\frac{g_i}{\|g\|_{(k,p),*}}\right)^{\frac{1}{p-1}}, \quad & \text{if } i \in \Ik \\
0, & \text{if } i \notin \Ik
\end{cases}
\end{equation}
and, for $p = \infty$ we choose the subgradient
\begin{equation}
\label{eq:FWinf}
s_i =
\begin{cases}
- \alpha \, {\rm sign}(g_i) & \text{if } i \in \Ik,~g_i \neq 0,\\
0 & \text{otherwise}.
\end{cases}
\end{equation}

\subsection{Projection Operator}
\label{sec:prox}
An alternative method to solve \eqref{eqn:ivanov} in the vector case is to consider the equivalent problem
\begin{align}
\min \left\{ f(w) +\delta_{ \{ \Vert \cdot \Vert_{(k,p)} \leq \alpha \} } (w): w \in \R^d \right\} \label{eqn:tikhonov},
\end{align}
where $\delta_C(\cdot)$ is the indicator function of convex set $C$. 
Proximal gradient methods can be used to solve optimization problems of the form $\min \left\{ f(w) + \lambda g(w): w \in \R^d\right\}$, where $f$ is a convex loss function with Lipschitz continuous gradient, $\lambda >0$ is a regularization parameter, and $g$ is a convex function for which the proximity operator can be computed efficiently see \cite{Beck2009,Nesterov2007} and references therein. 
The proximity operator of $g$ with parameter $\rho>0$ is defined as $\prox_{\rho g} (w) = \argmin \{\frac{1}{2} \Vert x-w\Vert^2 + \rho g(x) : x \in {\mathbb R}^d \}$. The proximity operator for the squared $k$-support norm was computed by \cite{Argyriou2012} and \cite{McDonald2014a}, and for the $k$-support norm by \cite{Chatterjee2014}. 

In the special case that $g(w) = \delta_C(w)$, where $C$ is a convex set, the proximity operator reduces to the projection operator onto $C$.  
For the $(k,p)$-support norm, for the case $p=\infty$ we can compute the projection onto its unit ball using the following result.

\begin{proposition}
\label{prop:projection-to-k-infinity-unit-ball}
For every $w\in \R^d$, the projection $x$ of $w$ onto the unit ball of the $(k,\infty$)-norm is given by 
\begin{align}
x_i = 
\begin{cases}
{\rm sign}(w_i) (\vert w_i \vert -\beta),	 &\text{if } \vert  \vert  w_i \vert -\beta \vert \leq 1,\\
{\rm sign}(w_i), 		 &\text{if } \vert  \vert  w_i \vert - \beta \vert  > 1,
\end{cases}\label{eqn:pinf-prox}
\end{align}
where $\beta=0$ if $\Vert w \Vert_1 \leq k$, otherwise $\beta \in (0,\infty)$ is chosen such that $\sum_{i=1}^d \vert x_i \vert = k$.
Furthermore, the projection can be computed in $\mathcal{O}(d \log d)$ time.
\end{proposition}

\begin{proof}(Sketch)
We solve the optimization problem
\begin{align}
\min_{x \in \R^d}  \left\{   \sum_{i=1}^d (x_i-w_i)^2: \vert x_i \vert \leq 1, \sum_{i=1}^d \vert x_i \vert \leq k     \right\}.\label{eqn:projection-problem-body}
\end{align}
We consider two cases.  If $\sum_{i=1}^d \vert w_i \vert \leq k$, then the problem decouples and we solve it componentwise.  
If $\sum_{i=1}^d \vert w_i \vert > k$, we solve problem \eqref{eqn:projection-problem-body} by minimizing the Lagrangian function $\mathcal{L}(x,\beta) = \sum_{i=1}^d (x_i-w_i)^2 +2 \beta ( \sum_{i=1}^d \vert x_i \vert - k)$ with nonnegative multiplier $\beta$. 
This can be done componentwise, and at the optimum the constraint $\sum_{i=1}^d  \vert x_i \vert \leq  k$ will be tight. 
Finally, both cases can be combined into the form of \eqref{eqn:pinf-prox}.
The complexity follows by taking advantage of the monotonicity of $x_i(\beta)$. 
\end{proof}
We can use Proposition \ref{prop:projection-to-k-infinity-unit-ball} to project onto the unit ball of radius $\alpha>0$ by a rescaling argument (see the appendix for details).

\subsection{Matrix Problem}
\label{sec:MP}
Given data matrix $X \in \R^{d \times m}$ for which we observe a subset of entries, 
we consider the constrained optimization problem
%
\begin{align}
\min_{W \in \R^{d \times m}} &\,  
\left\{ \Vert \Omega(X) - \Omega(W) \Vert_{\rm F}: \Vert W \Vert_{(k,p)} \leq \alpha \right\}
\label{eqn:matrix-optimization}
\end{align}
where the operator $\Omega$ applied to a matrix sets unobserved values to zero.
As in the vector case, the Frank-Wolfe method can be applied to the matrix problems. 
Algorithm \ref{alg:FW} is particularly convenient in this case as we only need to compute the largest $k$ singular values, which can result in a computationally efficient algorithm. 
The result is a direct consequence of Proposition \ref{prop:k-sup-p-norm-and-dual} and von Neumann's trace inequality, see e.g. \cite[Ch. 9 Sec. H.1.h]{Marshall1979}. 
We obtain that the solution of the inner minimization step is
$
U_k {\rm diag}(s) V_k\trans
$
where $U_k$ and $V_k$ are the top $k$ left and right singular vectors of the gradient $G$ of the objective function in \eqref{eqn:matrix-optimization} evaluated at the current solution, whose singular values we denote by $g$, and $s$ is obtained from $g$ as per equations \eqref{eq:FWp} and \eqref{eq:FWinf}, for $p\in (1,\infty)$ and $p=\infty$, respectively.

Note also that the proximity operator of the norm and the Euclidean projection on the associated unit ball both require the full singular value decomposition to be performed. Indeed, the proximity operator of an orthogonally invariant norm $\Vert \cdot \Vert = g( \sigma(\cdot))$ at $W \in \mathbb{R}^{d \times m}$ is given by $
\prox_{\Vert \cdot \Vert} (W)= U \text{diag}(\prox_{g}(\sigma(W))) V\trans$, where $U$ and $V$ are the matrices formed by the left and right singular vectors of $W$, see e.g. \cite[Prop. 3.1]{Argyriou2011}, and this requires the full decomposition.

\begin{table}[t!]
\caption{Matrix completion on rank 5 matrix with flat spectrum. The improvement of the $(k,p)$-support norm over the $k$-support and trace norms is considerable (statistically significant at a level $<0.001$). }
\vskip 0.1in
\label{table:synthetic-flat-p-infinity}
\centering
\setlength\tabcolsep{4pt}
\begin{minipage}[t]{0.98\linewidth}
\centering
\begin{small}
\begin{tabular}{llccc}
\toprule
   dataset & norm  & test error                           & $k$ & $p$  \\
\midrule
rank 5       & trace  & 0.8415 (0.03)                         & -   &-  \\ 
$\rho$=10\%  & k-supp  & 0.8343 (0.03)                          & 3.3 & - \\ 
             & kp-supp  & 0.8108 (0.05)                          & 5.0 & $\infty$ \\ 
\midrule
rank 5       & trace  & 0.6161 (0.03)                          &  -  & - \\ 
$\rho$=15\%  & k-supp  & 0.6129 (0.03)                          & 3.3 & - \\ 
             & kp-supp  & 0.4262 (0.04)                          & 5.0 & $\infty$ \\ 
\midrule
rank 5       & trace  & 0.4453 (0.03)             	            &-  & -   \\ 
$\rho$=20\%  & k-supp  & 0.4436 (0.02)                          & 3.5 & - \\ 
             & kp-supp  & 0.2425 (0.02)            	            & 5.0 & $\infty$  \\ 
\midrule
rank 5       & trace  & 0.1968 (0.01)             	            &-  & -   \\ 
$\rho$=30\%  & k-supp  & 0.1838 (0.01)                          & 5.0 & - \\ 
             & kp-supp  & 0.0856 (0.01)            	            & 5.0 & $\infty$  \\ 
\bottomrule
\end{tabular}
\end{small}
\end{minipage}
\end{table}

\section{Numerical Experiments}
\label{sec:experiments}
\vspace{-.1truecm}
In this section we apply the spectral $(k,p)$-support norm to matrix completion (collaborative filtering),   
in which we want to recover a low rank, or approximately low rank, matrix from a small sample of its entries, see e.g. \cite{Jaggi2010}.
One prominent method of solving this problem is trace norm regularization: we look for a matrix which closely fits the observed entries and has a small trace norm (sum of singular values) \cite{Jaggi2010,Mazumder2010,Toh2011}. 
We apply the $(k,p)$-support norm to this framework and we investigate the impact of varying $p$.  
Next we compare the spectral $(k,p)$-support norm to the trace norm and the spectral $k$-support norm ($p=2$) in both synthetic and real datasets.
In each case we solve the optimization problem \eqref{eqn:matrix-optimization} using the Frank-Wolfe method as outlined in Section \ref{sec:optimization}.
We determine the values of $k$ and $p \geq 1$ by validation, averaged over a number of trials. Specifically, we choose the optimal $p$, $k$, as well as the regularization parameter $\alpha$ by validation over a grid.
We let alpha vary in $10^0$ to $10^5$ with step $10^{0.25}$, we let $p$ vary over $20$ values from $1$ to $50,000$, plus $p=\infty$, and vary $k$ from $1$ to $20$. Our code is available from {\em http://www0.cs.ucl.ac.uk/staff/M.Pontil/software.html}.

\begin{figure}[t]
\vskip 0.2in
\begin{center}
\centerline{
\includegraphics[width=0.65\columnwidth]{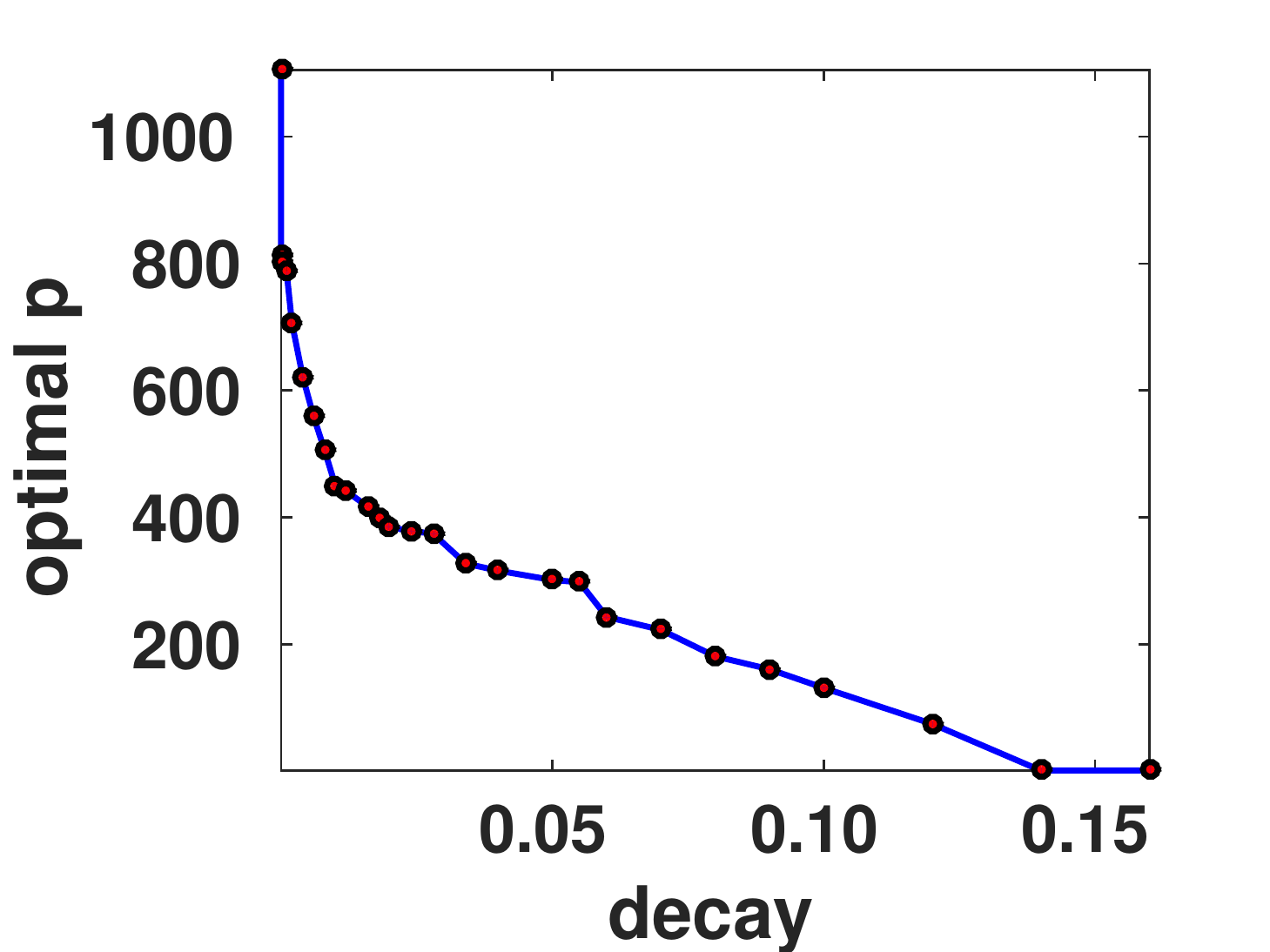}
}
\caption{Optimal $p$ vs. decay $a$.}
\label{fig:p-vs-decay-synthetic}
\end{center}
\vskip -0.2in
\end{figure}

\begin{figure}[t]
\vskip 0.2in
\begin{center}
\centerline{
\includegraphics[width=0.65\columnwidth]{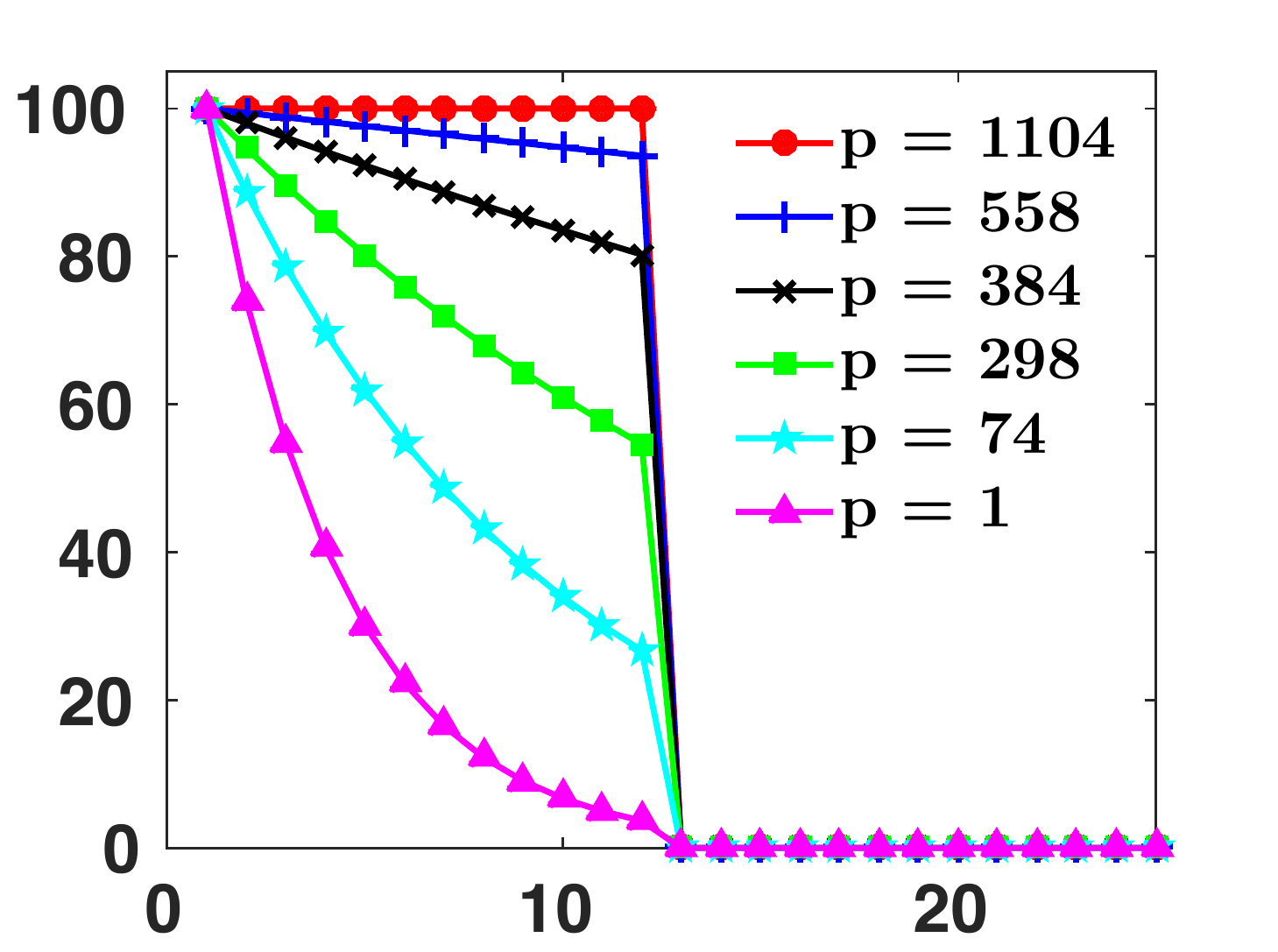}
}
\caption{Optimal $p$ fitted to Matrix spectra with various decays.}
\label{fig:p-vs-decay-real}
\end{center}
\vskip -0.2in
\end{figure}

\noindent {\bf Impact of $p$.}
A key motivation for the additional parameter $p$ is that it allows us to tune the norm to the decay of the singular values of the underlying matrix. 
In particular the variational formulation of \eqref{eqn:kp-sup-infimum-convolution} suggests that as the spectrum of the true low rank matrix flattens out, larger values of $p$ should be preferred. 
  
We ran the method on a set of $100 \times 100$ matrices of rank 12, with decay of the non zero singular values $\sigma_\ell$ proportional to $\exp(-\ell a)$, 
for 26 values of $a$ between $10^{-6}$ and $0.18$, and we determined the corresponding optimal value of $p$.
Figure \ref{fig:p-vs-decay-synthetic} illustrates the optimum value of $p$ as a function of $a$. 
We clearly observe the negative slope, that is the steeper the slope the smaller the optimal value of $p$.  
Figure \ref{fig:p-vs-decay-real} shows the spectrum and the optimal $p$ for several decay values. 
 
Note that $k$ is never equal to 1, which is a special case in which the norm is independent of $p$, and is equal to the trace norm. 
In each case the improvement of the spectral $(k,p)$-support norm over the $k$-support and trace norms is statistically significant at a level $<0.001$. 

Figure \ref{fig:error-vs-p-synthetic} illustrates the impact of the curvature $p$ on the test error on synthetic and real datasets.  
We observe that the error levels off as $p$ tends to infinity, so for these specific datasets the major gain is to be had for small values of $p$.  
The optimum value of $p$ for both the real and synthetic datasets is statistically different from $p=2$ ($k$-support norm), and $p=1$ (trace norm).

\begin{table}[t!]
\caption{
Matrix completion on real datasets. The improvement of the $(k,p)$-support norm over the $k$-support and trace norms is statistically significant at a level $<0.001$. }
\vskip 0.1in
\label{table:real-data}
\centering
\setlength\tabcolsep{4pt}
\begin{minipage}[th]{0.65\linewidth}
\centering
\begin{small}
\begin{tabular}{llccc}
\toprule
   dataset & norm  & test error                         & $k$ & $p$  \\
\midrule
MovieLens       & trace  & 0.2017                       & -   & -  \\ 
100k  		& k-supp  & 0.1990                          & 1.9  & - \\ 
             & kp-supp  & 0.1921                           & 2.0 & $\infty$ \\ 
\midrule
Jester 1      & trace  & 0.1752             	            &-  & -   \\ 
  			& k-supp  & 0.1739             	            & 6.4 & -   \\ 
			& kp-supp & 0.1744 & 2.0 & $\infty$  \\
             & kp-supp  & 0.1731             	            & 2.0 & 6.5  \\ 
\midrule
Jester 3      & trace  & 0.1959              	            &  -  & - \\ 
  			& k-supp  & 0.1942                          &  2.1 & - \\ 
             & kp-supp  & 0.1841                          & 2.0 & $\infty$ \\ 
\bottomrule
\end{tabular}
\end{small}
\end{minipage}
\end{table}

\noindent {\bf Simulated Data.}
Next we compared the performance of the $(k,p)$-support norm to that of the $k$-support norm and the trace norm for a matrix with flat spectrum.  As outlined above, as the spectrum of the true low rank matrix flattens out, larger values of $p$ should be preferred. 
Each $100 \times 100$ matrix is generated as $W = A S B\trans + E$, where $U$ and $V$ are the singular vectors of the matrix $U V\trans$, where $U,V \in \mathbb{R}^{100\times 5}$, the entries of $U$, $V$ and $E$ are i.i.d. standard Gaussian, and $S$ is diagonal with 5 non zero constant entries. 
Table \ref{table:synthetic-flat-p-infinity} illustrates the performance of the norms on a synthetic dataset of rank 5, with identical singular values, that is a flat spectrum.  
In each regime the case $p=\infty$ outperforms the other norms by a substantial margin, with statistical significance at a level $<0.001$. 
We followed the framework of \cite{McDonald2014a}  and use $\rho$ to denote the percentage of the data to use in the training set. 
We further replicated the setting of \cite{McDonald2014a} for synthetic matrix completion, and found that the $(k,p)$-support norm outperformed the standard $k$-support norm, as well as the trace norm, at a statistically significant level (see Table \ref{table:synthetic-mc} in the appendix for details).

We note that although Frank-Wolfe method for the $(k,p)$-support norm does not generally converge as quickly as proximal methods (which are available in the case of $k$-support norm \cite{McDonald2014,McDonald2014a,Chatterjee2014}), the computational cost can be mitigated using the continuation method.  Specifically given an ordered sequence of parameter values for $p$ we can proceed sequentially, initializing its value based on the previously computed value.  Empirically we tried this approach for a series of 30 values of $p$ and found that the total computation time increased only moderately.

\noindent {\bf Real Data.} Finally, we applied the norms to real collaborative filtering datasets.
We observe a subset of the (user, rating) entries of a matrix and predict the unobserved ratings, with the assumption that the true matrix is likely to have low rank. 
We report on the MovieLens 100k dataset ({\em http://grouplens.org/datasets/movielens/}), which consists
of ratings of movies, and
the Jester 1 and 3 datasets ({\em http://goldberg.berkeley.edu/jester-data/}), which consist of 
ratings of jokes. 
We followed the experimental protocol in \cite{McDonald2014a,Toh2011},
using normalized mean absolute error \cite{Toh2011}, and we implemented a final thresholding step as in \cite{McDonald2014a} (see the appendix for further details).
The results are outlined in Table \ref{table:real-data}.  
The spectral $(k,p)$-support outperformed the trace norm and the spectral $k$-support norm, and the improvement is statistically significant at a level $<0.001$ (the standard deviations, not shown here, are of the order of $10^{-5}$). 
In summary the experiments suggest that the additional flexibility of the $p$ parameter does allow the model to better fit both the sparsity and the decay of the true spectrum.

\begin{figure}[t]
\begin{center}
\centerline{
\includegraphics[width=0.65\columnwidth]{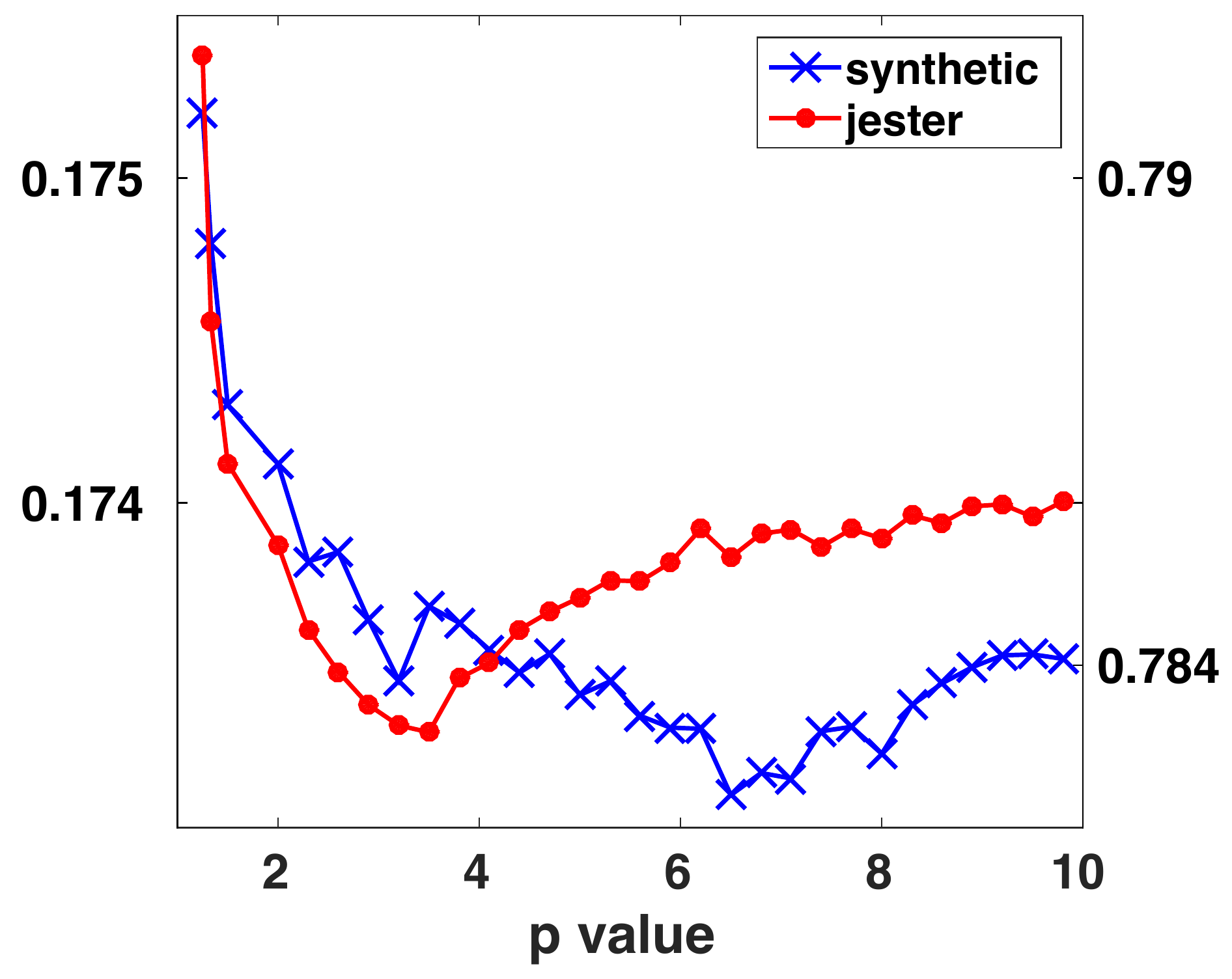}
}
\caption{Test error vs curvature ($p$).  Left axis: synthetic data (blue crosses); right axis: Jester 1 dataset (red circles).}
\label{fig:error-vs-p-synthetic}
\end{center}
\end{figure}

\section{Conclusion}
\label{sec:conclusion}
\vspace{-.1truecm}
We presented a generalization of the $k$-support norm, the $(k,p)$-support norm, where the additional parameter $p$ is used to better fit the decay of the components of the underlying model. We determined the dual norm, characterized the unit ball and computed an explicit expression for the norm.  
As the norm is a symmetric gauge function, we further described the induced spectral $(k,p)$-support norm. 
We adapted the Frank-Wolfe method to solve regularization problems with the norm, and in the particular case $p=\infty$ we provided a fast computation for the projection operator.
In numerical experiments we considered synthetic and real matrix completion problems and we showed that varying $p$ leads to significant  performance improvements.
Future work could include deriving statistical bounds for the performance of the norms, and situating the norms in the framework of other structured sparsity norms which have recently been studied.

\appendix
\section{Appendix}

In this appendix, we provide proofs of the results stated in the main body of the paper, and we include experimental details and results that were not included in the paper for space reasons.

\subsection{Proof of Proposition \ref{prop:k-sup-p-norm-and-dual}}
For every $u \in \R^d$ we have
\begin{align}
\Vert u \Vert_{(k,p),*} 
&= \max \left\{ \sum_{i=1}^d u_i w_i  : w \in C_k^p \right\} \nonumber \\
& = \max \left\{ \sum_{i=1}^d u_i w_i :  \card(w)\leq k, \Vert w \Vert_{p} \leq 1\right\} \nonumber \\
& = \max \left\{ \sum_{i\in \Ik} u_i w_i : \sum_{i \in \Ik} \vert w_i\vert^{p}  \leq 1 \right\}\nonumber  \nonumber \\
& = \left( \sum_{i \in \Ik} \vert u_i\vert^{q} \right)^{\frac{1}{q}},
\nonumber
\end{align}
where the first equality uses the definition of the unit ball \eqref{eq:unitball} and the second equality is true because the maximum of a linear functional on a compact set is attained at an extreme point of the set. The third equality follows by using the cardinality constraint, that is we set $w_i = 0$ if $i \notin \Ik$. Finally, the last equality follows by H\"older's inequality in $\R^k$ 
\cite[Ch. 16 Sec. D.1]{Marshall1979}. 

The second claim is a direct consequence of the cardinality constraint and H\"older's inequality in $\R^k$.
\qed

To prove Proposition \ref{prop:unit-ball-of-spectral} we require the following auxiliary result.  Let $X$ be a finite dimensional vector space. Recall that a subset $C$ of $X$ is called {\em balanced} if $\alpha C \subseteq C$ whenever $\vert \alpha \vert \leq 1$. 
Furthermore, $C$ is called {\em absorbing} if for any $x \in X$, $x \in \lambda C$ for some $\lambda >0$.  

\begin{lemma}\label{lem:minkowski-bounded}
Let $C\subseteq X$ be a bounded, convex, balanced, and absorbing set.  
The Minkowski functional $\mu_C$ of $C$, defined, for every $w \in X$, as
\begin{align*}
\mu_C(w) = \inf \left\{\lambda : \lambda >0, ~\frac{1}{\lambda} w \in C\right\}
\end{align*}
is a norm on $X$. 
\end{lemma}

\begin{proof}
We give a direct proof that $\mu_C$ satisfies the properties of a norm.  
See also e.g. \cite[\S 1.35]{Rudin1991} for further details.
Clearly $\mu_C(w) \geq 0$ for all $w$, and $\mu_C(0)=0$. Moreover, as $C$ is bounded,
$\mu_C(w) > 0$ whenever $w\ne 0$. 

Next we show that $\mu_C$ is one-homogeneous. For every $\alpha \in \R$, $\alpha \ne 0$, let $\sigma = {\rm sign}(\alpha)$ and note that
\begin{align*}
\mu_C(\alpha w) &= \inf \left\{\lambda>0 : \frac{1}{\lambda} \alpha w \in C \right\}\\
&= \inf \left\{\lambda>0 : \frac{|\alpha|}{\lambda}  \sigma w \in C \right\}\\
&= \vert \alpha \vert \inf \left\{\lambda>0 : \frac{1}{\lambda} 	w \in \sigma C \right\}\\
&= \vert \alpha \vert \inf \left\{\lambda>0 : \frac{1}{\lambda} 	w \in  C \right\}\\
&= \vert \alpha \vert \mu_C(w),
\end{align*}
where we have made a change of variable and used the fact that $\sigma C = C$.

Finally, we prove the triangle inequality. For every $v,w \in X$, if $v/\lam \in C$ and $w/\mu \in C$ then setting $\gamma =  \lam/(\lam+\mu)$, we have
$$
\frac {v+w}{\lam +\mu} =\gamma \frac{v}{\lam} + (1-\gamma)  \frac{w}{\mu}
$$
and since $C$ is convex, then $\frac {v+w}{\lam +\mu} \in C$. We conclude that $\mu_C(v+w) \leq \mu_C(v) + \mu_C(w)$. The proof is completed. 
\end{proof}
Note that for such set $C$, the unit ball of the induced norm $\mu_C$ is $C$. 
Furthermore, if $\|\cdot\|$ is a norm then its unit ball satisfies the hypotheses of Lemma \ref{lem:minkowski-bounded}.



\subsection{Proof of Proposition \ref{prop:unit-ball-of-spectral}}

Define the set
\begin{align*}
T_k^p &= \{W \in \mathbb{R}^{d\times m} : \textrm{rank}(W) \leq k, \Vert \sigma(W) \Vert_p \leq 1 \} , 
\end{align*}
and its convex hull $A_k^p = \textrm{co}(T_k^p)$, and consider the Minkowski functional 
\begin{align}
\lambda(W) 
&= \inf \{ \lambda>0: W \in \lambda A_k^p \}, ~~~ W\in \mathbb{R}^{d \times m}   \label{eqn:k-sup-matrix-def-2}.
\end{align}
We show that $A_k^p$ is absorbing, bounded, convex and symmetric, and it follows by Lemma \ref{lem:minkowski-bounded} that $\lambda$ defines a norm on $\R^{d \times m}$ with unit ball equal to $A_k^p$.
The set $A_k^p$ is clearly bounded, convex and symmetric.  
To see that it is absorbing, let $W$ in $\mathbb{R}^{d \times m}$ have singular value decomposition $U \Sigma V\trans$, and let $r=\min(d,m)$. 
If $W$ is zero then clearly $W\in A_k^p$, so assume it is non zero. 

{For $i \in \N_r$ let $S_i\in\R^{d\times m}$ have entry $(i,i)$ equal to $1$, and all remaining entries zero.
We then have
\begin{align*}
W&=  U \Sigma V\trans \\
&= U \left(\sum_{i=1}^r \sigma_i S_i  \right) V\trans \\
&= \left(\sum_{i=1}^d \sigma_i \right)   \sum_{i=1}^r \frac{\sigma_i}{\sum_{j=1}^r \sigma_j}(U S_i V \trans) \\
& =: \lambda \sum_{i=1}^r \lambda_i Z_i.
\end{align*}
Now for each $i$, $\Vert\sigma( Z_i) \Vert_p = \Vert \sigma(S_i) \Vert_p = 1$, and $\textrm{rank}(Z_i) = \textrm{rank}(S_i) = 1$, so $Z_i \in T_k^p$ for any $k \geq 1$. 
Furthermore $\lambda_i \in [0,1]$ and $\sum_{i=1}^r \lambda_i=1$, that is $(\lambda_1, \ldots, \lambda_r)$ lies in the unit simplex in $\R^d$,
so $\frac{1}{\lambda}W$ is a convex combination of elements of $Z_i$, in other words $W \in \lambda A_k^p$, and we have shown that $A_k^p$ is absorbing.
It follows that $A_k^p$ satisfies the hypotheses of Lemma \ref{lem:minkowski-bounded}, and $\lambda$ defines a norm on $\R^{d \times m}$ with unit ball equal to $A_k^p$.  }

Since the constraints in $T_k^p$ involve spectral functions, the sets $T_k^p$ and $A_k^p$ are invariant to left and right multiplication by orthogonal matrices. 
It follows that $\lambda$ is a spectral function, that is $\lambda(W)$ is defined in terms of the singular values of $W$.
By von Neumann's Theorem \cite{VonNeumann1937} the norm it defines is orthogonally invariant and we have
\begin{align}
\lambda(W) &= \inf \{ \lambda>0: W \in \lambda A_k^p \} \notag \\
&= \inf \{ \lambda >0: \sigma(W) \in \lambda C_k^p \}  \notag \\
&= \Vert \sigma(W) \Vert_{(k)} \notag
\end{align}
where we have used Equation \eqref{eq:unitball}, which states that $C_k^p$ is the unit ball of the $(k,p)$-support norm. 
It follows that the norm defined by \eqref{eqn:k-sup-matrix-def-2} is the spectral $(k,p)$-support norm with unit ball given by $A_k^p$.


\qed

\subsection{Proof of Proposition \ref{prop:projection-to-k-infinity-unit-ball}}

We solve the optimization problem
\begin{align}
\argmin_{x \in \R^d}  \left\{   \sum_{i=1}^d (x_i-w_i)^2: \vert x_i \vert \leq 1, \sum_{i=1}^d \vert x_i \vert \leq k     \right\}.\label{eqn:projection-problem}
\end{align}
We consider two cases.  If $\sum_{i=1}^d \vert w_i \vert \leq k$, then the problem decouples and we solve it componentwise.  
Specifically we minimize $(x_i-w_i)^2$ subject to $ \vert x_i \vert \leq 1$, and the solution is immediately given by
\begin{align}
x_i = 
\begin{cases}
-1, \quad &\text{if } w_i  < -1,\\
w_i, \quad &\text{if } -1 \leq w_i \leq 1,\\
1, \quad &\text{if }   w_i  >1.
\end{cases}\label{eqn:pinf-1}
\end{align}
We now assume that  $\sum_{i=1}^d \vert w_i \vert > k$.  
Consider the Lagrangian function $\mathcal{L}(x,\beta) = \sum_{i=1}^d (x_i-w_i)^2 +2 \beta \left( \sum_{i=1}^d \vert x_i \vert - k\right)$ with nonnegative multiplier $\beta$.  
We solve problem \eqref{eqn:projection-problem} by minimizing the Lagrangian with respect to $x$, which can be done componentwise due to the coupling effect of the Lagrangian.  
Furthermore, at the optimum the constraint $\sum_{i=1}^d  \vert x_i \vert \leq  k$ will be tight. 
The derivative with respect to $x_i$ is zero when $x_i = w_i - \beta \, \text{sign}(x_i)$. 
Incorporating the constraint $\vert x_i \vert \leq 1$ we get the following solution
\begin{align}
x_i = 
\begin{cases}
-1, 			\quad &\text{if } w_i + \beta <  -1,\\
w_i +\beta,	\quad &\text{if } -1 \leq w_i + \beta \leq 0,\\
w_i -\beta,	\quad &\text{if }  0 \leq w_i - \beta  \leq 1,\\
1, 			\quad &\text{if }  w_i - \beta  > 1,
\end{cases}\label{eqn:pinf-2}
\end{align}
where $\beta \geq 0$ is chosen such that $\sum_{i=1}^d \vert x_i(\beta) \vert = k$.
Note that for $\beta=0$, which corresponds to $\Vert w\Vert_1 \leq k$, \eqref{eqn:pinf-2} reduces to \eqref{eqn:pinf-1}, hence we obtain the compact notation \eqref{eqn:pinf-prox}.
Finally, note that the expression $\sum_{i=1}^d \vert x_i(\beta) \vert$ decreases monotonically as $\beta$ increases.  
In the case that $\Vert w \Vert_1 > k$, $\beta \in (0, \vert w_j \vert -1) $, where $\vert w_j \vert = \argmin \vert w_i \vert$, hence we can determine $\beta$ by binary search in $\mathcal{O}(d \log d)$ time. 

\qed

In order to project onto the unit ball of radius $\alpha>0$, we solve the optimization problem 
$\min \{   \sum_{i=1}^d (x_i-w_i)^2~:~x \in \R^d,~\vert  x_i \vert  \leq \alpha, ~\sum_{i=1}^d \vert x_i \vert \leq \alpha k  \}$. 
To do so, we make the change of variables $x_i' = x_i/\alpha$ and note that the problem reduces to computing the projection $x'$ of $w'$ onto the unit ball of the norm, where $w_i' = w_i/\alpha$, which is the problem 
that was solved in \eqref{eqn:projection-problem-body}.
Once this is done, our solution is given by $x_i=\alpha x_i'(\beta)$, where $x'(\beta)$ is determined in accordance with Proposition \ref{prop:projection-to-k-infinity-unit-ball}.

\subsection{Numerical Experiments}
In this section we report further experimental details and results not included in the main body of the paper for space reasons.

{\bf Simulated Datasets.}
We replicated the setting of \cite{McDonald2014a} in order to verify that the additional parameter can improve performance. 
Each $100 \times 100$ matrix is generated as $W=U V\trans +E$, where $U,V \in \mathbb{R}^{100\times r}$, $r \ll 100$, and the entries of $U$, $V$ and $E$ are i.i.d. standard Gaussian. 
Table \ref{table:synthetic-mc} illustrates the results. 
The error is measured as $\Vert \textrm{true} -  \textrm{predicted}\Vert^2 / \Vert \textrm{true} \Vert^2$, standard deviations are shown in brackets and the mean values of $k$ and $p$ are selected by validation.  
We note that the spectral $(k,p)$-support norm outperforms the standard spectral $k$-support norm in each regime, and the improvement is statistically significant at a level $<0.01$.

\begin{table}[t]
\caption{Matrix completion on synthetic datasets generated with decaying spectrum. The improvement of the $(k,p)$-support norm over the $k$-support and trace norms is statistically significant at a level $<0.001$. }
\vskip 0.1in
\label{table:synthetic-mc}
\centering
\setlength\tabcolsep{4pt}
\begin{minipage}[th]{0.98\linewidth}
\centering
\begin{small}
\begin{tabular}{llccc}
\toprule
   dataset & norm  & test error                           & $k$ & $p$  \\
\midrule
rank 5       & trace & 0.8184 (0.03)                         & -   &-  \\ 
$\rho$=10\%  & k-supp  & 0.8036 (0.03)                         &  3.6  & - \\ 
             &  kp-supp  & 0.7831 (0.03)                          & 1.8 & 7.3 \\ 
\midrule
rank 5       & trace  & 0.4085 (0.03)                         &  -  & - \\ 
$\rho$=20\%  & k-supp  & 0.4031 (0.03)                          &  3.1  & - \\ 
             & kp-supp  & 0.3996 (0.03)                          & 2.0 & 4.7 \\ 
\midrule
rank 10      & trace  & 0.6356 (0.03)             	            &-  & -   \\ 
$\rho$=20\%  & k-supp  & 0.6284 (0.03)            	            & 4.4 & -   \\ 
             & kp-supp  & 0.6270 (0.03)            	            & 2.0 & 4.4  \\ 
\bottomrule
\end{tabular}
\end{small}
\end{minipage}
\end{table}

{\bf Real Datasets.} 
The MovieLens 100k dataset ({\em http://grouplens.org/datasets/movielens/}) consists of 943 user ratings of 1,682 movies, the ratings are integers from 1 to 5, and all users have rated a minimum number of 20 films. 
The Jester 1 dataset ({\em http://goldberg.berkeley.edu/jester-data/}) consists of ratings of 24,983 users of 100 jokes, and the Jester 3 dataset consists of ratings of 34,938 users of 100 jokes, and the ratings are real values from $-10$ to $10$.

Following \cite{McDonald2014a,Toh2011}, for MovieLens for each user we uniformly sampled $\rho=50\%$ of available entries for training, and for Jester 1 and Jester 3 we sampled 20, respectively 8 ratings per user, using 10\% for validation.
We used normalized mean absolute error, 
$$
\textrm{NMAE}=\frac{\Vert \text{true}-\text{predicted}\Vert^2}{\#\text{obs.}/(r_{\max} -r_{\min})},
$$
where $r_{\min}$ and $r_{\max}$ are lower and upper bounds for the ratings \cite{Toh2011}, 
and we implemented a final thresholding step as in \cite{McDonald2014a}.

\end{document}